%% file: aaai22.tex
\def\year{2022}\relax
\documentclass[letterpaper]{article} 
\usepackage{aaai22}  
\usepackage{times}  
\usepackage{helvet}  
\usepackage{courier}  
\usepackage[hyphens]{url}  
\usepackage{graphicx} 
\usepackage{natbib}  
\usepackage{caption} 
\DeclareCaptionStyle{ruled}{labelfont=normalfont,labelsep=colon,strut=off} 
\usepackage{algorithm}
\usepackage{algorithmic}

\usepackage{newfloat}
\usepackage{listings}
\floatstyle{ruled}
\newfloat{listing}{tb}{lst}{}
\floatname{listing}{Listing}
\usepackage{booktabs}
\usepackage{amsmath}
\usepackage{amssymb}
\usepackage{amsthm}
\usepackage{amsfonts}
\usepackage{bm}
\usepackage{bbm}
\usepackage{thmtools} 
\usepackage{thm-restate}
\newtheorem{lemma}{Lemma}
\newtheorem{remark}{Remark}
\newtheorem{theorem}{Theorem}
\newtheorem{definition}{Definition}

\newtheorem{corollary}{Corollary}
\DeclareMathOperator*{\argmax}{arg\,max}

\usepackage{algorithm}
\usepackage{algorithmic}
\usepackage{caption}
\usepackage{subcaption}

\usepackage{xr}

\makeatletter
\newcommand*{\addFileDependency}[1]{
\typeout{(#1)}
\@addtofilelist{#1}
\IfFileExists{#1}{}{\typeout{No file #1.}}
}\makeatother

\newcommand*{\myexternaldocument}[1]{%
\externaldocument[app:]{#1}%
\addFileDependency{#1.tex}%
\addFileDependency{#1.aux}%
}
\myexternaldocument{supplement}

\title{Differentiable Good Arm Identification}

\author {
    Yun-Da Tsai, 
    Tzu-Hsien Tsai, 
    Shou-De Lin 
}
\affiliations {
    f08946007@csie.ntu.edu.tw\\
    alan070516180339887@gmail.com\\
    sdlin@csie.ntu.edu.tw
}

\usepackage{bibentry}

\begin{document}

\maketitle

\begin{abstract}

This paper focuses on a variant of the stochastic multi-armed bandit problem known as good arm identification (GAI). GAI is a pure-exploration bandit problem that aims to identify and output as many good arms as possible using the fewest number of samples. A good arm is defined as an arm whose expected reward is greater than a given threshold.
We present our study in the context of a structured bandit setting and introduce DGAI, a novel, differentiable good arm identification algorithm. By leveraging a data-driven approach, DGAI significantly enhances the state-of-the-art HDoC algorithm empirically.  Additionally, we demonstrate that DGAI can improve the cumulative reward maximization problem when a threshold is provided as prior knowledge for the arm set.
Extensive experiments have been conducted to validate our algorithm's performance. The results demonstrate that our algorithm significantly outperforms the competitors in both synthetic and real-world datasets for both the GAI and MAB tasks.

\end{abstract}

\section{Introduction}
Bandit problems represent a category of sequential decision-making problems where stochastic rewards are observed as feedback. Within this broad context, the classic multi-armed bandit (MAB) problem aims to maximize the expected cumulative reward over a series of trials by navigating the exploration-exploitation dilemma.
Another essential category within the bandit problem is the best arm identification (BAI), a pure-exploration problem in which the agent aims to identify the best arm with minimum sample complexity~\cite{kalyanakrishnan2012pac,audibert2010best,kano2019good}. This paper's specific focus is on a derivative of the pure-exploration problem called the good arm identification (GAI) problem~\cite{kano2019good}.

In the GAI problem, a "good arm" is an arm whose expected reward meets or exceeds a given threshold. The agent must identify such an arm within a specified error probability during repeated trials, stopping only when it believes that no good arms remain. The goal of GAI is to identify as many arms as possible with as few samples as possible. Key differences between GAI and BAI include:
1. GAI identifies arms through absolute comparisons (i.e., above a threshold) rather than relative (best) comparisons, requiring distinct optimization strategies.
2. Unlike BAI, which examines the entire arm set, GAI can be solved by an anytime algorithm. This flexibility allows GAI to provide reasonable solutions even if interrupted, making it suitable for time-sensitive applications like online advertising and high-speed trading~\cite{kano2019good}.
3. GAI handles the exploration-exploitation dilemma with confidence. Exploitation entails pulling the arm most likely to be good, while exploration means pulling other arms to build confidence in identifying them as either good or bad.
The overall sample complexity is bounded by union bounds $\Delta_i = |\mu_i - \epsilon|$ (i.e., the gap between thresholds for identification) and $\Delta_{i,j} = \mu_i - \mu_j$ (i.e., the gap between arms for sampling) where $\Delta = \min\{\min_{i\in[K]} \Delta_i, \min_{\lambda\in[K-1]} \Delta_{\lambda,\lambda+1}/2\}$ (see Table~\ref{tab:notation} for notation explanations).

Common GAI algorithms like HDoC, LUCB-G, and APT-G~\cite{kano2019good} encompass a sampling strategy and an identification criterion (see Algorithm~\ref{alg:hdoc}). The former determines which arm to draw, while the latter decides whether to accept an arm as good. Once an arm is deemed good, it is removed from the pool, as there is no need for further sampling.
These strategies formulate confidence bounds through concentration inequalities based on assumptions about the reward distribution like sub-Gaussianity. Though essential for theoretical analysis, these bounds may not always align with practical applications.
Constructed confidence bounds tend to be conservative in real-world situations~\cite{osband2015bootstrapped,kveton2019garbage}, as they're based on assumed reward distributions, leading to broad, non-adaptive confidence bounds and sub-optimal performance. Our research highlights the superiority of adaptive, data-driven techniques.

This work introduces the GAI problem within a structured bandit setting at the intersection of multi-armed bandits and learning halfspaces.
We aim for a data-driven approach to learn the confidence bound, ensuring adaptability to the actual problem structure.
Furthermore, we propose that confidence bounds for both sampling and identification in GAI should be learned separately in a data-driven manner to reflect different distributions of gaps between arms and thresholds (the union bounded sample complexity).
We introduce an algorithm called Differentiable GAI (or DGAI) that employs a differentiable UCB index to accomplish these objectives.
This index offers more information than the classical UCB index and can be integrated with a smoothed indicator function to learn the confidence bound for both the sampling strategy and identification criteria.
By using the newly designed differentiable algorithm and learning objective function, the adaptive confidence bound significantly enhances empirical performance over existing GAI algorithms in both online and offline settings, where the latter can learn GAI hyperparameters over multiple training trajectories, revealing the converged optimal confidence bound.

Next, we demonstrate that DGAI can be extended to enhance cumulative reward maximization problems with a given threshold as prior knowledge. Generally, the criteria for eliminating sub-optimal arms in MAB and bad arms in GAI are the same when the conventional UCB algorithm is chosen for sampling, as shown in Sec.~\ref{sec:mab_with_thres}. By employing the DGAI algorithm, we facilitate dynamic and adaptive criteria to discard sub-optimal arms in a data-driven manner, with the threshold as prior information guiding the learning and adjustment of the confidence bound. This reformulation enables us to achieve both the maximization of cumulative reward and the elimination of sub-optimal arms, contributing to further empirical performance improvements.


In summary, this paper makes the following contributions:
\begin{enumerate}

    \item We introduce a novel structured bandit problem, linear \& non-linear GAI, closing a gap in the literature on structured bandits.
    \item We propose DGAI, a novel differentiable algorithm designed to learn adaptive confidence bounds based on actual reward structures rather than assumptions.
    \item We provide correctness guarantees establishing that linear DGAI is $\delta$-PAC.
    \item We show that DGAI outperforms state-of-the-art baselines on both synthetic and real-world datasets for GAI problems.
    \item We demonstrate that incorporating DGAI can improve cumulative reward maximization problems.
\end{enumerate}

\section{Related works}

\subsection{Good arm identification}
The work by Kano et al.~\cite{kano2019good} was the first to formulate the Good Arm Identification (GAI) problem as a pure-exploration problem in the fixed confidence setting. In the fixed confidence setting, an acceptance error rate (confidence) is fixed, and the goal is to minimize the number of arm pulls required to identify good arms (sample complexity).
This work addresses a new kind of dilemma: the exploration-exploitation dilemma of confidence. Here, exploration involves the agent pulling arms other than the currently best one to increase the confidence of whether that arm is good or bad. On the other hand, exploitation entails the agent pulling the currently best arm to increase confidence in its goodness.
To tackle this problem, the authors propose three algorithms: a hybrid algorithm for the Dilemma of Confidence (HDoC), the Lower and Upper Confidence Bounds algorithm for GAI (LUCB-G), which builds upon the LUCB algorithm for best arm identification~\cite{kalyanakrishnan2012pac}, and the Anytime Parameter-free Thresholding algorithm for GAI (APT-G), based on the APT algorithm for the thresholding bandit problem~\cite{locatelli2016optimal}. The lower bound on the sample complexity for GAI is proven to be $\Omega(\lambda \log \frac{1}{\delta})$, and HDoC can identify $\lambda$ good arms within $O(\lambda \log \frac{1}{\delta} + (K - \lambda) \log \log \frac{1}{\delta})$ samples.

\subsection{Differentiable bandit algorithm}
SoftUCB~\cite{yang2020differentiable} solves the policy-gradient optimization of bandit policies via a differentiable bandit algorithm.
The authors present a differentiable UCB-typed linear bandit algorithm that combines EXP3~\cite{auer1995gambling} and Phased Elimination~\cite{lattimore2020bandit} to allow the confidence bound to be learned purely in a data-driven fashion and avoid the need for reliance on concentration inequalities and assumptions about the reward distribution's form.
By incorporating a differentiable UCB index, the learned confidence bound becomes adaptable to the problem's actual reward structure.
Additionally, the authors proved that the differentiable UCB index has a regret bound that scales with $d$ and $T$ as $\mathcal{O}(\beta\sqrt{dT})$, compared to classical UCB-typed algorithms, where $d$ is the dimension of the input to the linear bandit model, and $T$ is the total number of sampling rounds.

\section{Preliminary}
In this section, we first formulate the GAI problem in the fixed confidence setting and show lower bound on sample complexity for GAI.
Next we will introduce the basic framework for GAI algorithm that follows the HDoC algorithm.
We give the notations in Table~\ref{tab:notation}.

\begin{table}[h]
    \centering
    \begin{tabular}{ |p{1cm}||p{6cm}|}
    \hline
        $K$ & Number of arms.\\
        $m$ & number of good arms (unknown).\\
        $A$ & set of arms, where $|A|=K$.\\
        $\xi$ & Threshold determining whether a arm is good or bad.\\
        $\delta$ & acceptance error rate \\
        $\mu_i$ & The true mean of $i^{th}$ arm.\\
        $\hat{\mu}_{i, t}$ & The empirical mean of $i^{th}$ arm at time $t$.\\
        $\tau_{\lambda}$ & The number of rounds to identify whether $\lambda^{th}$ arm is good or arm.\\
        $\tau_{stop}$ & Round that agent confirms no good arm left.\\
        $N_i(t)$ & The number of samples of arm $i$ by the end of sampling round $t$.\\
        $U_{i,t}$ & confidence bound at time $t$.\\
        $\overline{\mu}_{i,t}$ & upper confidence bound.\\
        $\underline{\mu}_{i,t}$ & lower confidence bound.\\
        $\|x\|_M$ & $\sqrt{x^TMx}$, $x\in \mathbb{R}^d$.\\
        $\Delta_i$ & $ |\mu_i - \xi|$ gap between thresholds for identification.\\
        $\Delta_{i,j}$ & $\mu_i - \mu_j$ gap between arms for sampling.\\
    \hline
    \end{tabular}
\caption{Notations}
\label{tab:notation}
\end{table}

\subsection{Good arm identification}
Let $K$ be the number of arms, $\xi$ be the threshold, and $\delta$ be the acceptance error rate. The reward of arm $i$, where $i \in [1, \dots, K]$, follows the mean $\mu_i$, which is initially unknown. We define "good" arms as those arms whose means are larger than the given threshold $\xi$. Without loss of generality, we assume the following order for the means:
\begin{equation}
    \mu_1 \geq \mu_2 \geq \dots \geq \mu_{m} \geq \xi \geq \mu_{m + 1} \geq \dots \geq \mu_K   
\end{equation}
Here, $m$ represents the total number of good arms, which is also unknown to the agent.

In each round $t$, the agent selects arm $a(t)$ to pull and receives a reward. Based on the previous rewards this arm has received, the agent classifies it as either a good or bad arm. The agent continues this process until all arms are identified as either good or bad arms, and this stopping time is denoted as $\tau_{stop}$. The agent's outputs are denoted as $\hat{a}_1, \hat{a}_2, \dots$, representing the identification of good arms at rounds $\tau_1, \tau_2, \dots$, respectively.

\subsection{Sample Complexity for GAI}
The lower bound on the sample complexity for GAI is given in~\cite{kano2019good} in terms of top-$\lambda$ expectations $\{\mu_i\}_{i=1}^{\lambda}$ and is tight up to the logarithmic factor $O(\log \frac{1}{\delta})$.
\begin{definition}[$(\lambda, \delta)$-PAC]\label{lambdaPAC}
An algorithm satisfying the following conditions is called ($\lambda$, $\delta$)-PAC:
if there are at least $\lambda$ good arms then
$\mathbb{P}[\{\hat{m} < \lambda\}\,\cup\,\bigcup_{i\in  \{\hat{a}_1, \hat{a}_2, \dots, \hat{a}_\lambda \}  }\{\mu_i < \xi\}]\le \delta$
and if there are less than $\lambda$ good arms then
$\mathbb{P}[ \hat{m}\ge \lambda] \le \delta$.
\end{definition}
{\allowdisplaybreaks[0]
\begin{theorem}
    Under any $(\lambda, \delta)$-PAC algorithm, if there are $m \ge \lambda$ good arms, then
    \begin{align}
        \mathbb{E}[\tau_{\lambda}]
         &\ge
         \left(
             \sum_{i=1}^{\lambda}\frac{1}{d(\mu_i,\xi)}\log\frac{1}{2\delta}
         \right)
         -
         \frac{m}{d(\mu_{\lambda},\xi)},
    \end{align}
    where $d(x,y) = x\log(x/y)+(1-x)\log((1-x)/(1-y))$ is the binary relative entropy, with convention that $d(0,0)=d(1,1)=0$.
\end{theorem}
}

\subsubsection{Hybrid algorithm for the Dilemma of Confidence (HDoC)}
HDoC's sampling strategy is based on the upper confidence bound (UCB) algorithm for the cumulative regret minimization~\cite{auer2002finite}, and its identification criterion (i.e., to output an arm as a good one) is based on the lower confidence bound (LCB)  identification~\cite{kalyanakrishnan2012pac}.
The upper bound on the sample complexity of HDoC is given in the following corollary based on Theorem 2 in~\cite{kano2019good}:

\begin{corollary}
    Let $\Delta$=$\min\{\min_{\lambda \in [K-1]} \Delta_{\lambda,\lambda+1}/2,\min_{i\in[K]}\Delta_i\}$.
    Then, for any $\lambda \le m$,
    \begin{align}
    &\mathbb{E}[\tau_{\lambda}]=\mathcal{O}\left(
    \frac{\lambda \log \frac{1}{\delta} +(K-\lambda)\log\log \frac{1}{\delta}+K\log \frac{K}{\Delta}}{\Delta^2}
    \right)\\
    &\mathbb{E}[\tau_{\mathrm{stop}}]=\mathcal{O}\left(
    \frac{K \log (1/\delta)+K\log (K/\Delta) }{\Delta^2}
    \right)
    \end{align}
\label{coro:GAI-upper-bound}
\end{corollary}

Note that evaluating the error probability in HDoC relies on the union bound over all rounds $t\in \mathbb{N}$, and its use does not worsen the asymptotic analysis for $\delta\to 0$.
However, the empirical performance can be considerably improved by avoiding the union bound and using an identification criterion based on such a bound.
The complexity measure for the sampling strategy is based on the gap between arms $\Delta_{i,j}$, and the complexity measure for the identification criteria is based on the gap between the threshold $\Delta_i$.

\input{LaTeX/hdoc-algo}

\section{Algorithm}
In this section, we proposed the differential GAI algorithm, DGAI, and proved DGAI is $\delta$-PAC.
Next, we show DGAI in linear form and later extend to non-linear cases.
Last we introduced the differential UCB index and followed by the training objectives.
Pseudo codes of all the algorithms are described
in Algorithm~\ref{alg:ours}.

\subsection{Linear Good Arm Identification}
First, we introduce the DGAI in linear bandit matrix forms.
Each arm $i \in \mathcal{A}$ is associated with a known feature vector $\mathbf{x}_i \in \mathbb{R}^d$.
Linear bandit is equivalent to discrete and independent arms using one-hot feature vectors.
The expected reward of each arm $\mu_i=\mathbf{x}_i^T\boldsymbol{\theta}$ follows a linear relationship over  $\mathbf{x}_i$ and an unknown parameter vector $\boldsymbol{\theta}$.
At each round $t$, the confidence bound is defined as 
$|\hat{\mu}_{i,t}-\mu_i|\leq \beta||\mathbf{x}_i||_{\mathbf{V}^{-1}_t}, \ \ \forall i\in \mathcal{A}$ and $\mathbf{V}_t=\sum_{t=1}^T \mathbf{x}_t \mathbf{x}_t^T$ is the Gram matrix up to round $t$.
$\beta$ in the confidence bound can be derived based on different concentration
inequalities under the assumption of the stochastic reward structure.
e.g., Hoeffding inequality~\cite{hoeffding1994probability}, self-normalized~\cite{abbasi2011improved}, Azuma Inequality~\cite{lattimore2018bandit}, Bernstein inequality~\cite{mnih2008empirical}.
The correctness of the output arms can be verified by the following theorem, whose proof is given in Appendix~\ref{proof:PAC}.

\begin{theorem}
    The DGAI algorithm is $\delta$-PAC which outputs a bad arm with probability at most $\delta$.
\label{thm:PAC}
\end{theorem}

\subsection{Differentiable UCB index}
In DGAI, our objective is to learn the magnitude of the confidence bound in a data-driven manner without relying on prior assumptions about the unknown reward distribution. The design of our UCB index and associated lemmas primarily follows~\cite{yang2020differentiable}.
Let the arm with the largest lower confidence bound at round $t$ as $i_*=\arg\max_{i \in\mathcal{A}}\hat{\mu}_{i,t}-U_{i,t}$.
We then define:
\begin{align}
\phi_{i,t}  = U_{i,t} + U_{i_*,t}  \text{\ \  and \ \  }   \hat{\Delta}_{i,t}=\hat{\mu}_{i_*,t}-\hat{\mu}_{i,t} 
\label{eq:diff-comb}
\end{align}
and $\hat{\Delta}_{i,t}$ is the estimated reward gap between $i_*$ and $i$.
The  UCB index $S_{i,t}$ is defined as 
\begin{equation}
\label{eq:S}
    S_{i,t}=\beta\phi_{i,t}-\hat{\Delta}_{i,t}\,
\end{equation}

\begin{lemma}
The UCB index $S$ provides the following information:
$S_{i,t}<0$, i.e., $\mu_*-\mu_i>0$, implies arm $i$ is a sub-optimal arm. $S_{i,t}\geq S_{j,t}\geq 0$, i.e., $\hat{\mu}_{i,t}+U_{i,t}\geq \hat{\mu}_{j,t}+U_{j,t}$, implies arm i has larger upper confidence bound.
\end{lemma}

For each round $t \in [T]$, the likelihood of selecting arm $i$ is: 
\begin{equation}
\label{eq:P}
    p_{i,t}=\frac{\exp(\gamma_t S_{i,t})}{\sum_{j=1}^K \exp(\gamma_t S_{j,t})}~,~ \gamma_t=\frac{\log\left(\frac{\delta |\mathcal{L}_t|}{1-\delta}\right)}{\tilde{S}_{\text{max},t}}
\end{equation}

where $\gamma_t>0$ is the coldness parameter modulating the concentration of the policy distribution, as elucidated in Lemma~\ref{lemma:gamma_lemma}.
The set $\mathcal{L}_t$ consists of suboptimal arms (i.e., $S_{t}<0$), while $\mathcal{U}_t$ holds non-suboptimal arms (i.e., $S_{t}\geq 0$). We have  $\mathcal{U}_t\cup \mathcal{L}_t=\mathcal{A}$ and $\mathcal{U}_t\cap \mathcal{L}_t=\emptyset$.
$\tilde{S}_{\text{max},t}=\max_{i \in \mathcal{U}_t}S_{i,t}$,   $|\mathcal{L}_t|$ denotes the cardinality of $\mathcal{L}_t$ and $\delta$ is a probability hyper-parameter.
\begin{lemma}
At any round $t\in [T]$, for any $\delta\in (0,1)$, setting $\gamma_t\geq \log(\frac{\delta|\mathcal{L}_t|}{1-\delta})/\tilde{S}_{\text{max},t}$ guarantees that $p_{\mathcal{U}_t}=\sum_{i\in \mathcal{U}_t}p_{i,t}\geq \delta $ and 
$p_{\mathcal{L}_t}=\sum_{i\in \mathcal{L}_t}p_{i,t}<1-\delta$ such that sub-optimal arms ($i \in L_t$) will be selected with near zero probability when $\delta \approx 1$.
\label{lemma:gamma_lemma}
\end{lemma}

It is worth noting that DGAI can readily be extended to non-linear versions. Appendix~\ref{sec:nonlinear} provides details on these adaptations. Remark~\ref{rmk:kernel} describes the adaptation to the kernelized DGAI variant by incorporating KernelUCB~\cite{valko2013finite}, and Remark~\ref{rmk:neural} showcases the transition to a neural network version of DGAI by incorporating NeuralUCB~\cite{zhou2020neural}. A more comprehensive experiment and theoretical analysis are reserved for future studies.

\subsection{Training Objectives}
The optimization objective for sampling strategy and identification criteria are independent and learned separately.
The optimization objective of the sampling strategy is to maximize the cumulative reward with constraints shown as the following differentiable function over $\beta$:
\begin{equation}
\label{eq:offline_sample_objective}
\begin{split}
    \max_{\beta}\sum_{t=1}^T\mathbb{E}[y_t]
    =\max_{\beta}\sum_{t=1}^T \sum_{i=1}^K p_{i,t}\mu_i~,
    \ \ \\s.t. \  |\mu_i-\hat{\mu}_{i,t}| \lessapprox \beta||\mathbf{x}_i||_{\mathbf{V}^{-1}_t},
    \ \ \forall i\in \mathcal{A}, t\in [T]
\end{split}
\end{equation}
The addition constraint force to tighten the upper confidence bound as low as possible and ensures that $\beta$ is minimal but $\beta||\mathbf{x}_i||_{\mathbf{V}^{-1}_t}$ is still indeed an actual upper confidence bound at any round $t \in [T]$ for any arm $i \in \mathcal{A}$.
Applying the Lagrange multipliers gives the objective with $\eta$ as the hyperparameter for balancing the constraints:
\begin{equation}
\label{eq:offline_sample_objective_constraint}
\begin{split}
   \max_{\beta}\sum_{t=1}^T \sum_{i=1}^K p_{i,t}\mu_i-\eta_1 C - \eta_2 |C|~, \ \ s.t. \ \ \eta>0\\
   \text{\ \  where \ \  }
   C = |\mu_i-\hat{\mu}_{i,t}|-\beta||\mathbf{x}_i||_{\mathbf{V}^{-1}_t}
\end{split}
\end{equation}

The objective for the identification criteria is to maximize the \textit{Exploit score} which is defined as follow:
\begin{equation}
\label{eq:exploit_score}
    \sum_{i=1}^K (T-t_i) * (\mu_i - \xi)
\end{equation}
where $t_i$ is the round that arm $i$ is identified.
The exploit score gives a positive reward when the arm is identified correctly and increases if the arm is identified earlier.
This conditional reward is formed as an indicator function with the lower confidence bound in the objective over $\alpha$:
\begin{equation}
\label{eq:offline_identify_objective}
\begin{split}
    \max_{\alpha}\sum_{t=1}^T \sum_{i=1}^K \mathbbm{I}[\mu_{i,t} - \alpha \|x_i\|_{V_t^{-1}} > \xi] * (\mu_{i, t} - \xi)\\
    s.t. |\mu_{i,t} - \xi| \gtrapprox \alpha \|x_i\|_{V_t^{-1}},
    \ \ \forall i\in \mathcal{A}, t\in [T]
\end{split}
\end{equation}

The indicator function is relaxed and smoothen with sigmoid function and again the Lagrange multipliers yields the following objective:
\begin{multline}
\label{eq:offline_identify_objective_constraint}
   \max_{\beta}\sum_{t=1}^T \sum_{i=1}^K \sigma((\mu_{i,t} - \alpha \|x_i\|_{V_t^{-1}} - \xi) * M) * (\mu_{i, t} - \xi)
   \\
   -\eta_1 D - \eta_2 |D|~, \ \ s.t. \ \ \eta>0~, \ \ \text{where~}
   \\
   D = \alpha||\mathbf{x}_i||_{\mathbf{V}^{-1}_t}-|\mu_i-\hat{\mu}_{i,t}|~, \ \ \sigma(x)=\frac{1}{1+e^{-x}}
\end{multline}
$\eta$ is for balancing the constraints and $M$ is for tuning the sharpness of the indicator function.

\input{LaTeX/diff-algo}

\subsection{Training Settings}

We trained our differentiable algorithm in both offline and online settings.
The online setting is more suitable for bandit problems, while the offline setting shows the converged optimal confidence bound and would be the upper bound of the online setting.

\subsubsection{Offline setting}
In the offline setting, we will run multiple epochs of $T$-rounds training trajectories with the identical arm set $\mathcal{A}$ to train $\alpha$ and $\beta$.
First, $\alpha_0$ and $\beta_0$ are initialized.
Both parameters are used to run one entire $T$-rounds training trajectories.
After the trajectory, $\alpha_0$ and $\beta_0$ is updated with Equation~\ref{eq:offline_sample_objective_constraint} and Equation~\ref{eq:offline_identify_objective_constraint} respectively.
After $N$ epochs of trajectories, we have our finalized confidence bound with learned $\alpha_N$ and $\beta_N$.
DGAI for learning the identification criterion in the offline setting is shown in Algorithm~\ref{alg:ours}.
The time complexity for training in the offline setting is $\mathcal{O}(NKT)$.

\subsubsection{Online setting}
In this setting, $\alpha$ and $\beta$ are updated online in one single $T$-rounds trajectory.
First, $\alpha_0$ and $\beta_0$ are initialized and for every $b\in[T]$ round, $\alpha$ and $\beta$ will perform batch update where $b$ is the batch size.

We adopted policy gradient methods, the same as non-episodic reinforcement learning problems~\cite{sutton2018reinforcement} to update the parameters in an online fashion.
Instead of maximizing the cumulative reward until the end of the trajectory, we update the observed cumulative reward up to round $t$ and bootstrapped future reward under the current policy $\boldsymbol{\pi}_t=[p_{1,t}, p_{2,t},..., p_{K,t}]$.
The online objective for sampling and identification are as follows, where $R$ is the reward feedback we obtain in the original objective function Equation~\ref{eq:offline_sample_objective} and Equation~\ref{eq:exploit_score}:
\begin{equation}
\label{eq:online_objective}
\begin{split}
    \left(\sum_{s=1}^t\sum_{i=1}^K R_{i,s}+(T-t)\sum_{i=1}^K R_{i,t}\right)/T
\end{split}
\end{equation}

The time complexity for training in the online setting is ${O}(KT)$.


\input{LaTeX/exp1-performance-fig}
\input{LaTeX/curve-fig}
\input{LaTeX/bound-fig}
\input{LaTeX/exp2-performance-fig}

\subsection{GAI with Cumulative Reward Maximization}
\label{sec:mab_with_thres}
Classical bandit problems focus on the cumulative reward maximization problem, where the objective is to maximize the sum of the samples collected by the algorithm up to time $T$~\cite{auer1995gambling}.
In this present section, we delve into transforming the "exploration-exploitation dilemma of confidence" that GAI addresses back into a cumulative reward maximization problem.
Specifically, we tackle the cumulative reward maximization by incorporating the additional information of a threshold as prior knowledge of the arm set $A$ allowing us to utilize this threshold to optimize and adjust the confidence bound to maximize cumulative reward.
In identical settings, \cite{pmlr-v48-locatelli16} utilized the APT algorithm and provided a regret bound with optimal threshold $\xi = \mu_*$ as illustrated in theorem~\ref{thm:apt}.
The bandit policy performs more efficiently in terms of regret if the threshold $\xi$ is closer to $\mu_*$ and achieves optimality when $\xi = \mu_*$.
However, since $\mu_*$ is always unknown in reality, the proper threshold must often be estimated beforehand, relying on prior experience.
\begin{theorem}\label{thm:apt}
Let $K > 0, R > 0$ and $T \geq 2K$ and consider a problem where the distribution of the arms is R-sub-Gaussian.
Run with parameters $\xi = \mu^*$ such that
\begin{align*}
Regret(T) = T\mu^* - \mathbb E \sum_{t \leq T} \mu_t  &\leq \inf_{\delta \geq 1} \Big[\sum_{k \not = k^*} \frac{4R^2\log(T)\delta}{\mu^* - \mu_i}\\ 
&+ (\mu^*- \mu_i)(1 + \frac{K}{T^{2\delta - 2}})\Big].
\end{align*}
\end{theorem}

In what follows, we incorporate DGAI to learn a dynamic and adaptive threshold rather than making random assumptions about the threshold. The upper confidence bound used for sampling serves as an identification criterion; for instance, $| \hat{\mu}_{i, t} - \mu_i | \leq U$ implies arm $i$ is good $\mu_i \geq \xi$ when $\hat{\mu}_{i, T_i(t)} - U \geq \xi$.
We have implemented two distinct confidence bounds to aid in solving the cumulative reward maximization problem. The first confidence bound, drawn from Equation~\ref{eq:offline_sample_objective}, directs the sampling strategy to select arms based on criteria that softly eliminate sub-optimal options.
In contrast, the other confidence bound, derived from Equation~\ref{eq:offline_identify_objective}, selects arms according to criteria that reject those falling below the threshold.
The objective function to solve the multi-arm bandit with threshold re-formulates both Equation~\ref{eq:offline_sample_objective} and Equation~\ref{eq:offline_identify_objective} into a single, unified objective function:
\begin{equation}
\label{eq:combine_objective}
\begin{split}
    \max_{\alpha,\beta}\sum_{t=1}^T\mathbb{E}[y_t]
    =\max_{\alpha,\beta}\sum_{t=1}^T \sum_{i=1}^K p_{i,t}\mu_i~,\\
    p_{i, t} = \frac{exp(S_{i,t}*I^{\xi}_{i,t})}{\sum exp(S_{j,t} * I^{\xi}_{i,t})}~, ~~~~~~~~~~ \\
    I_{i, t} = \sigma((\mu_{i,t} - \alpha \|x_i\|_{V_t^{-1}} - \xi) * M)
\end{split}
\end{equation}


\section{Experiment}
\label{sec:exp}

Here we would like to evaluate DGAI in (1) improvements in GAI problem (2) improvements in cumulative reward maximization with given threshold.

\subsection{Experiment Setup}

\subsubsection{Datasets}
We use two synthetic and two real-world datasets whose details are listed in Appendix Table~\ref{tab:dataset}.
We generated one small and one large synthetic dataset with a number of arms = 50 and 1000, respectively. The reward of each arm is generated from the uniform distribution $\in [0.5005, 0.49975]$, and the threshold $\xi$ is set to 0.5.
For real world datasets,  we used two public recommendation datasets where each item can be treated as an item: OpenBandit~\cite{saito2020large} and MovieLens~\cite{harper2015movielens}. We use the method similar to \cite{da2022fast,tsai2024lil}, to transform multi--class dataset to bandit dataset
 In the MovieLens dataset, the rating of each movie is treated as a means to sample the stochastic reward of each arm. 
The threshold for a good arm is set to the 95 percentile reward of all arms in the entire dataset.

\begin{table}
    \centering
    \tabcolsep=0.15cm
    \begin{tabular}{lcccccccc}
    \toprule
        ~& synth & synth & MovieLens & Openbandit\\
        ~& small & large & ~ & ~\\
    \midrule
        HDoC & 1.2e6 & 6.5e4 & 1.8e6 & 4e5\\
        LUCB-G & 1.4e6 & 1.1e5 & 1.7e6 & 4e5\\
        APT-G & 1.8e6 & 8.8e4 & 6e6 & 4e5\\
        TT-TS & 1.7e6 & 1.1e5 & 4.1e7 & 4e5\\
        SoftUCB-G & 2.7e6 & 1.2e5 & 3.8e7 & 1e6\\
        DGAI-online & \textbf{3.1e6} & \textbf{1.7e5} & \textbf{4.2e7} & \textbf{1.3e6}\\
        DGAI-offline & \textbf{4.1e6} & \textbf{2.5e5} & \textbf{6.1e7} & \textbf{1.5e6}\\
    \bottomrule
    \end{tabular}
    \caption{\textit{Exploit score} of DGAI and the baselines. The result shows that DGAI outperforms the baselines in all cases. Results exhibit a standard deviation below 10\%, averaged over 10 repetitions per run.}
    \label{tab:exp1}
\end{table}


\subsubsection{Parameter setting}
For the two optimized parameters $\alpha, \beta$, we set initial values to 0.
We used stochastic gradient descent as the optimizer and 1e-1 as the learning rate.
$\eta_1$ and $eta_2$ in Equation~\ref{eq:offline_sample_objective_constraint} and Equation~\ref{eq:offline_identify_objective_constraint} are both set to 1e-3.
It is worth mentioning that fine-tuning their parameters can improve the convergence of the training, and a bad hyperparameter may not converge to optimal in one single training trajectory.

\subsection{Good arm identification}
\subsubsection{Evaluation Metric}
The metric for evaluating the improvement of good arm identification problem is by the \textit{Exploit score} $\sum_{i=1}^K (T-t_i) * (\mu_i - \xi)$ which is also used in Equation~\ref{eq:exploit_score}.
The increase in the cumulative score from this metric indicates that the agent outputs good arms faster.
It not only evaluates the overall performance (when the agent finishes identifying all arms) but also considers each individual arm (when each arm is identified).

\subsubsection{Baseline}
We compare with three SOTA algorithms: HDoC, APT-G, LUCB-G introduced in~\cite{kano2019good} for solving the good arm identification problem.
We also compare with top-two Thompson sampling (TT-TS)~\cite{russo2016simple}, which is a Bayesian method designed for arm identification.
We also added SoftUCB-G, which runs the exact HDoC algorithm, except that we replaced the conventional upper confidence bound with SoftUCB~\cite{yang2020differentiable} such that it also learns the upper bound with differential data-driven methods.
The results of all baselines are averaged with ten different random seeds.

\subsubsection{Sampling Bound}
\begin{enumerate}
    \item HDoc : Pull arm $\hat{a}^* = \argmax_{i\in A} \overline{\mu}_{i,t}$~, where\\
    $\overline{\mu}_{i,t} = \hat{\mu}_{i,t} + \sqrt{\frac{\log{t}}{2N_i(t)}}$
    \item LUCB-G : Pull arm $\hat{a}^* = \argmax_{i\in A} \overline{\mu}_{i,t}$~, where\\
    $\overline{\mu}_{i,t} = \hat{\mu}_{i,t} + \sqrt{\frac{\log{4KN_i^2{t}/\delta}}{2N_i(t)}}$
    \item APT-G : Pull arm $\hat{a}^* = \argmax_{i\in A} \beta_{i,t}$~, where\\
    $\beta_{i,t} = \hat{\mu}_{i,t} + \sqrt{N_i(t)}|\xi - \hat{\mu}_{i,t}|$
    \item TT-TS : Pull arm $\hat{a}^* = \argmax_{a\not\in I_m} P_{i,t}$, where $I = \argmax_{a\in A} P_{i,t}$ and $P$ is the posterior probability of arm $i$ is optimal.
    \item SoftUCB-G : Pull arm $\hat{a}^* = \argmax_{i\in A} P_{i,t}$
    \item DGAI : Pull arm $\hat{a}^* = \argmax_{i\in A} P_{i,t}$
\end{enumerate}

\subsubsection{Identification Bound}
All baseline have identical identification bound:
$\sqrt{\frac{\log{4KN_i^2(t)/\delta}}{2N_i(t)}}$ and DGAI is $\alpha||\mathbf{x}_i||_{\mathbf{V}^{-1}_t}$.

\subsection{Cumulative Reward Maximization}

\subsubsection{Evaluation Metric}
We use the cumulative reward (same as the classic multi-arm bandit problem) to evaluate the results. \textbf{Baseline}
We compare our methods to classical UCB, Thompson sampling and SoftUCB.

\subsection{Results}
The experiment results for DGAI are shown in Table.~\ref{tab:exp1}.
The complete performance through the training epochs for GAI is shown in Fig.~\ref{fig:exp1}.
The cumulative exploit score for DGAI outperforms other baselines eventually for both online and offline learning settings.
Fig.~\ref{fig:learning-curve} depicts the learning curves of $\alpha$ and $\beta$.
It shows that the confidence bound in both online and offline settings converges in a similar trend, but the parameters converge must faster in the offline setting since it updates the parameters with multiple complete training trajectories.
We can observe that the confidence bound w.r.t $\alpha,\beta$ continually adjusts during the single round training trajectory and, at last, converge to a scale similar to in the offline setting.
In Fig.~\ref{fig:bound-compare}, we plot the confidence bound for the best arm in the arm set $\mathcal{A}$ during the training trajectory and compare the identification bounds between DGAI and the baselines.
We can observe that as the training epochs go on, the learned confidence bound w.r.t $\alpha,\beta$ continually converges to optimal, and thus the cumulative exploit score continues to increase.
The results suggested that the learned confidence bound with DGAI is significantly tighter than the original union confidence bound consistently in all cases.
The experiment results solving the cumulative reward maximization problem with threshold are shown in Fig.~\ref{fig:exp2}.
We can observe that DGAI outperforms other baselines on all datasets with the help of the additional identification criteria based on the threshold.
Note that the curves in all figures are smoothed to ignore the vibration.

\section{Conclusion}

In this research, we have presented a novel differentiable algorithm for the good arm identification (GAI) problem within the framework of a structured bandit setting.
Our algorithm, DGAI, has the unique capability to adapt the confidence bound, learning through gradient ascent, which sets it apart from existing methods.
We showed that DGAI can be applied in both online/offline settings and could be further extended to non-linear settings.
Furthermore, we also show that DGAI improves cumulative reward maximization problems when a threshold is provided as prior knowledge.
The experiments on all datasets have shown order-of-magnitude improvements compared to other baselines in terms of rewards obtained.
A more comprehensive experiment for the generalized linear case and theoretical analysis for sample complexity is reserved for future studies.

\clearpage
\bibliography{aaai22}
\clearpage
\appendix

\newcommand{\nn}{\nonumber\\}
\newcommand{\since}[1]{\quad\mbox{#1}}
\newcommand{\omu}{\overline{\mu}}
\newcommand{\umu}{\underline{\mu}}

\section{Proof of Theorem~\ref{thm:PAC}}
\label{proof:PAC}

  \begin{proof}[Proof of Theorem \ref{thm:PAC}] 
    We show that DGAI is $(\lambda,\delta)$-PAC for arbitrary $\lambda\in[K]$.
    
    First we consider the case that there are more than or equal to
    $\lambda$ good arms and show
    \begin{align}
        \mathbb{P}\left[
        \{\hat{m}< \lambda\}\,\cup\,\bigcup_{i\in  \{\hat{a}_1, \hat{a}_2, \ldots, \hat{a}_\lambda \}  }\{\mu_i< \xi\}
        \right]\le \delta
        \label{pac_toprove1}.
    \end{align}
    
    Since we are now considering the case $m\ge \lambda$,
    the event $\{\hat{m}< \lambda\}$ implies that
    at least one good arm $j\in[m]$ is regarded as a bad arm, that is,
    $\{\umu_{j,n} \le \xi\}$ occurs for some $j\in [m]$ and $n\in\mathbb{N}$.
    Thus we have
    \begin{align}
    \mathbb{P}[\hat{m}< \lambda]
    &\le
    \sum_{j\in [m]}
        \mathbb{P}\left[\bigcup_{n \in \mathbb{N}}\{\omu_{j,n} < \xi\} \right]
    \nn
    &\le
    \sum_{j\in [m]}
    p_{i, t}
    \nn
    &\le
    \delta\since{by Lemma \ref{lemma:gamma_lemma}}\label{eq_error_pac}
    \end{align}
    On the other hand, since the event $\bigcup_{i\in  \{\hat{a}_1, \hat{a}_2, \ldots, \hat{a}_\lambda \}  }\{\mu_i< \xi\}$
    implies that $j \in \{\hat{a}_i\}_{i=1}^{\lambda}$ for some
    bad arm $j\in [K]\setminus [m]$,
    we have
    \begin{align}
    \mathbb{P}\left[\bigcup_{i\in  \{\hat{a}_1, \hat{a}_2, \ldots, \hat{a}_\lambda \}  }\{\mu_i< \xi\}\right]
    &\le
    \sum_{j\in [K]\setminus [m]}\mathbb{P}[j \in \{\hat{a}_i\}_{i=1}^{\lambda}]\nn
    &\le
    \sum_{j\in [K]\setminus [m]}
        \mathbb{P}\left[\bigcup_{n \in \mathbb{N}}\{\umu_{j,n} \ge \xi\} \right]
    \nn
    &\le
    \delta\label{eq_error_pac2}
    \end{align}
    in the same way as \eqref{eq_error_pac}.
    We obtain \eqref{pac_toprove1} by putting \eqref{eq_error_pac} and \eqref{eq_error_pac2} together.

    Next we consider the case that
    the number of good arms $m$ is less than $\lambda$
    and show
    \begin{align}
    \mathbb{P}[\hat{m}\ge \lambda]\le \delta\,.\label{pac_toprove2}
    \end{align}
    Since there are at most $m<\lambda$ good arms,
    the event $\{\hat{m}\ge \lambda\}$ implies that
    $j \in \{\hat{a}_i\}_{i=1}^{\lambda}$ for some $j\in [K]\setminus[\lambda]$.
    Thus, in the same way as \eqref{eq_error_pac2} we have
    \begin{align}
    \mathbb{P}\left[\hat{m}\ge \lambda\right]
    &\le
    \sum_{j\in [K]\setminus [m]}\mathbb{P}[j \in \{\hat{a}_i\}_{i=1}^{\lambda}]\nn
    &\le
    \delta\,,\nonumber
    \end{align}
    which proves \eqref{pac_toprove2}.
\end{proof}


\section{Extension to non-linear DGAI}
\label{sec:nonlinear}
In this section, we will show how DGAI can be easily extended to non-linear cases, specifically, kernelised and neural network versions. Note that further experiments and theoretical analysis will be left for future work.

\begin{remark}
\label{rmk:kernel}

We can easily extend DGAI in linear form to non-linear kernelised version.
In the following, we applied the kernel trick~\cite{cristianini2004kernel} and the kernelised version of the Mahalanobis~\cite{haasdonk2010classification} similar to the settings in KernelUCB~\cite{valko2013finite}.
The derivation is straightforward and we provide it for convenience and to introduce the notation.
Kernel methods assume that there exists a mapping $\phi : \mathbb{R}^d \rightarrow H$ that maps the data to a (possibly infinite dimensional) Hilbert space in which a linear relationship can be observed.
We call $\mathbb{R}^d$ the primal space and $H$ the associated reproducing kernel Hilbert space (RKHS).
We use matrix notation to denote the inner product of two elements $h, h' \in H$, i.e. $h^Th' := \langle h, h' \rangle_H$ and $\|h\| = \sqrt{\langle h, h \rangle_H}$ to denote the RKHS norm.
From the mapping $\phi$ we have the kernel function, defined by: $k(x, x') := \phi(x)^T\phi(x')$, $\forall x, x' \in \mathbb{R}^d$, and the kernel matrix of a data set $\{x_1, . . . , x_t\} \subset \mathbb{R}^d$ given by $K_t := \{k(x_i, x_j)\}_{i,j\leq t}$.
Following the above we obtain:

\begin{equation}
\hat{\mu}_{i,t} = k_{x,t}^T(K_t + \gamma I)^{-1}y_t.
\end{equation}

\begin{equation}
    \hat{U}_{i,t} := \gamma^{-1/2}\sqrt{k(x_{i,t}, x_{i,t}) - k_{x,t}^T(K_t + \gamma I)^{-1}k_{x,t}}.
\end{equation}

where $k_{x,t} := \Phi_t\phi(x) = [k(x, x_1), . . . , k(x, x_{t-1})]^T$ and for some regularization parameters $\gamma > 0$. Plugging this into equation~\ref{eq:diff-comb} and~\ref{eq:S}, we can construct the kernelised version of differentiable UCB index.
\end{remark}

\begin{remark}
\label{rmk:neural}
We can easily extend DGAI in linear form to non-linear neural network version.
In the following, we use a neural network $f(x; \theta)$ to predict the reward of context, and upper confidence bounds computed from the network to guide exploration similar to NeuralUCB~\cite{zhou2020neural}.
A simplified version of NeuralUCB where only the first-order Taylor approximation of the neural network around the initialized parameter is updated through online ridge regression  can be seen as KernelUCB specialized to the Neural Tangent Kernel~\cite{jacot2018neural}, or LinUCB~\cite{li2010contextual} with Neural Tangent Random Features~\cite{cao2019generalization}.
Following the above we obtain:
Denote the gradient of the neural network function by $g(x; \theta) = \nabla_{\theta} f(x; \theta) \in \mathbb{R}^p$, $Z = \gamma \textbf{I}$ and network width $M$. We obtain:
\begin{equation}
    \mu_i = f(x_{i,t}; \theta_{t-1})
\end{equation}
\begin{equation}
    \hat{U}_{i,t} = \beta_{t-1}\sqrt{g(x_{t,a}; \theta_{t-1})^\top Z_{t-1}^{-1} g(x_{t,a}; \theta_{t-1})/M}
\end{equation}

where $g(x; \theta) = \nabla_{\theta} f(x; \theta) \in \mathbb{R}^p$ is the gradient of the neural network, $Z = \gamma \textbf{I}$ and $M$ is the network width.
Plugging this into equation~\ref{eq:diff-comb} and~\ref{eq:S}, we can construct the neural network version of differentiable UCB index.

\end{remark}




\begin{table}[h!]
    \centering
    \begin{tabular}{lcccccccc}
    \toprule
        ~& $N$ arms & $T$ trajectory & $\xi$ threshold\\
    \midrule
        synth small & 50 & 1e6 & 0.5\\
        synth large & 1000 & 1e6 & 0.5\\
        MovieLens & 9527 & 1e5 & 0.071\\
        Openbandit & 80 & 107 & 0.005\\
    \bottomrule
    \end{tabular}
    \caption{Details of the datasets.}
    \label{tab:dataset}
\end{table}

    

\end{document}

%% file: LaTeX/hdoc-algo.tex
\renewcommand{\algorithmicrequire}{\textbf{Input:}}

\begin{algorithm}[t]

\caption{HDoC algorithm. The base algorithm for solving GAI problem.}

\begin{algorithmic}[1]
\label{alg:hdoc}
\REQUIRE a  $\xi$, $\delta$, $K$, $\mathcal{A}$
\STATE Pull each arm once
\FOR{$t=1$ to $T$}
    \STATE Pull arm $a^*_t$ selected by sampling strategy
    \IF{$\underline{\mu}_{a^*_t} \geq \xi$}
        \STATE Identify $a^*_t$ as good
        \STATE remove $a^*_t$ from $\mathcal{A}$
    \ENDIF
    \IF{$\overline{\mu}_{a^*_t} < \xi$}
        \STATE Identify $a^*_t$ as bad
        \STATE remove $a^*_t$ from $\mathcal{A}$
    \ENDIF
\ENDFOR

\end{algorithmic}
\end{algorithm}

%% file: LaTeX/diff-algo.tex
\renewcommand{\algorithmicrequire}{\textbf{Input:}}
\renewcommand{\algorithmicensure}{\textbf{Initialize:}}

\begin{algorithm}[t]

\caption{The differential algorithm to learn the optimal confidence bound for both sampling strategy and identification criteria.}

\begin{algorithmic}[1]
\label{alg:ours}

\REQUIRE $N$, $\mathcal{A}$, $K$, $T$,  $\lambda$, $\eta$
\ENSURE  $\alpha_0=0,~\beta_0=0,~V_0=I\in\mathbb{R}^{d\times d},~b_0\in\mathbb{R}^d,~\gamma_0=0$

\FOR{$n=1$ to $N$}
    \FOR{$t=1$ to $T$}
        \STATE Find $S_{i,t}, \forall i\in \mathcal{A}$ via Eq.~\ref{eq:S} with $\beta$.
        \STATE Find $P_t$ via Eq.~\ref{eq:P}.
        \STATE Select arm $i_t \in \mathcal{A}$ randomly according to $P_t$ and receive reward $y_t$.
        \STATE Update $\mathbf{V}_t\gets \mathbf{V}_{t}+\mathbf{x}_t\mathbf{x}_t^T$, $\mathbf{b}_t\gets \mathbf{b}_{t-1}+\mathbf{x}_ty_t$.
        \STATE Update $\gamma_t$ via Eq.~\ref{eq:P}.
        \STATE (Online) Update $\beta \gets \hat{\beta}_{t-1}+\lambda \nabla_\beta$, $\alpha \gets \hat{\alpha}_{t-1}+\lambda \nabla_\alpha$ via Eq.~\ref{eq:online_objective}
    \ENDFOR
    \STATE (Offline) Update $\beta \gets \hat{\beta}_{t-1}+\lambda \nabla_\beta$, $\alpha \gets \hat{\alpha}_{t-1}+\lambda \nabla_\alpha$ via Eq.~\ref{eq:offline_sample_objective_constraint} and Eq.~\ref{eq:offline_identify_objective_constraint}
\ENDFOR
\end{algorithmic}
\end{algorithm}

%% file: LaTeX/exp1-performance-fig.tex
\begin{figure*}[t]
    \centering
    \begin{subfigure}[b]{0.24\textwidth}
        \centering
        \includegraphics[width=\textwidth]{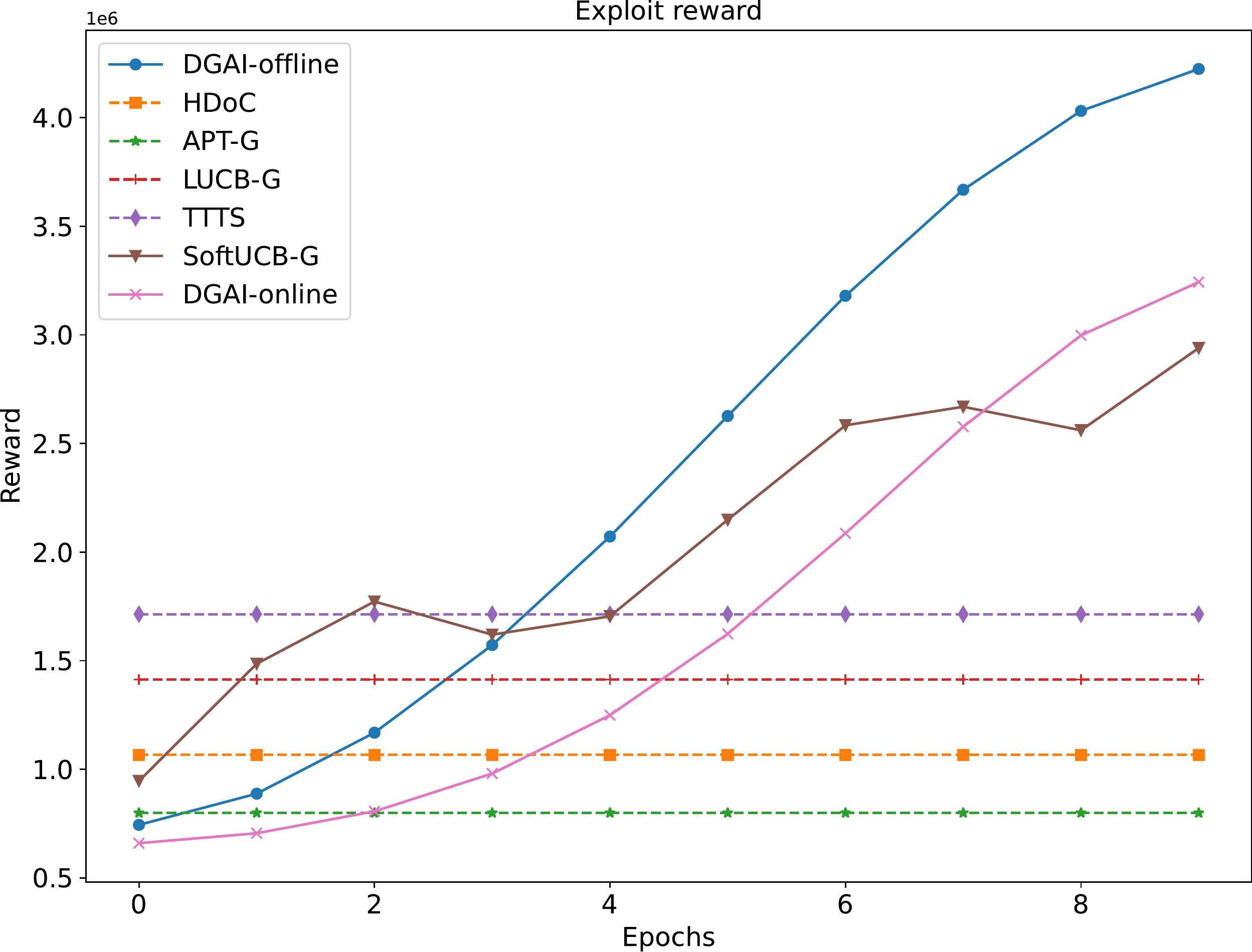}
        \caption{Synth small}
    \end{subfigure}
    \hfill
    \begin{subfigure}[b]{0.24\textwidth}
        \centering
        \includegraphics[width=\textwidth]{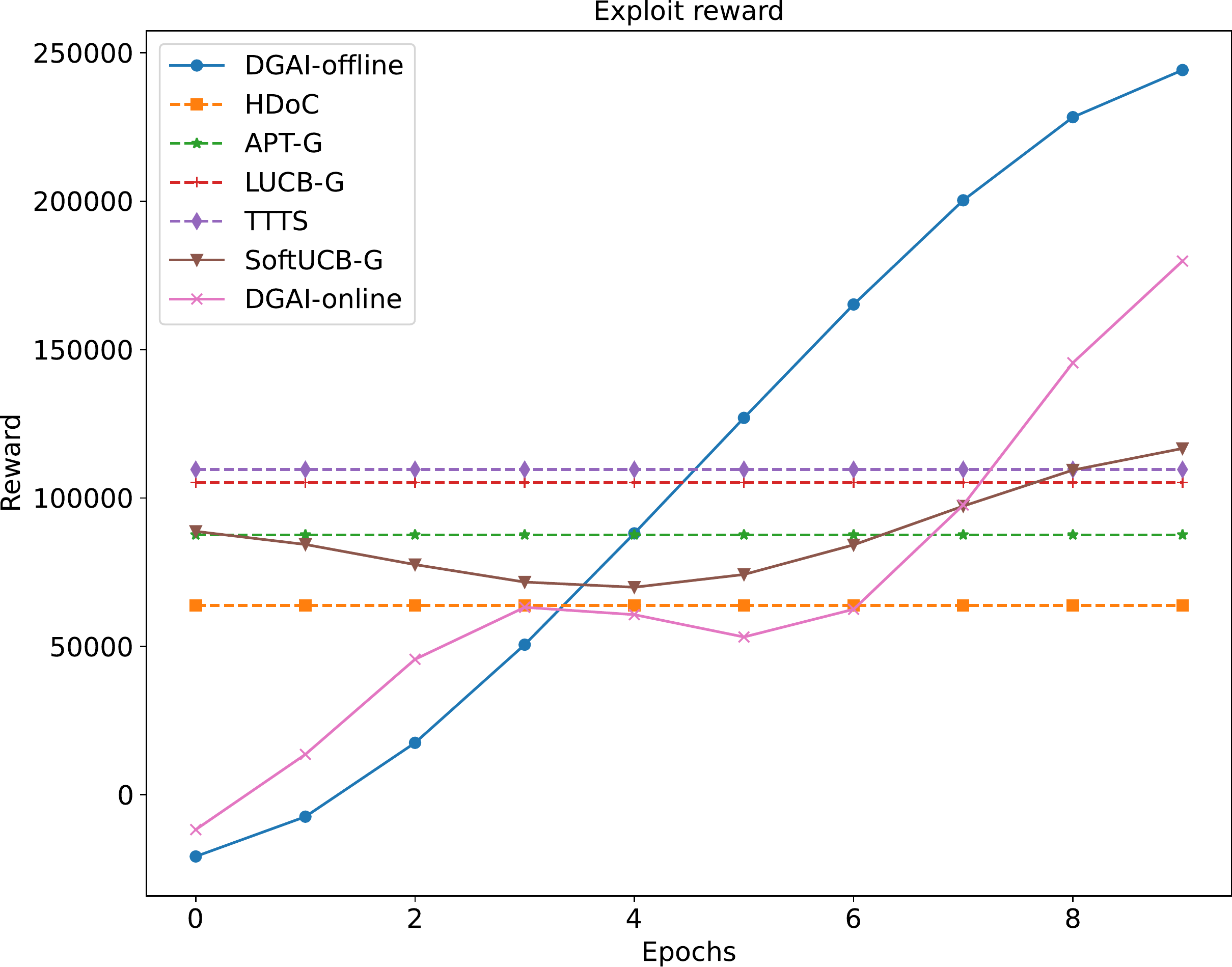}
        \caption{Synth large}
    \end{subfigure}
    \hfill
    \begin{subfigure}[b]{0.24\textwidth}
        \centering
        \includegraphics[width=\textwidth]{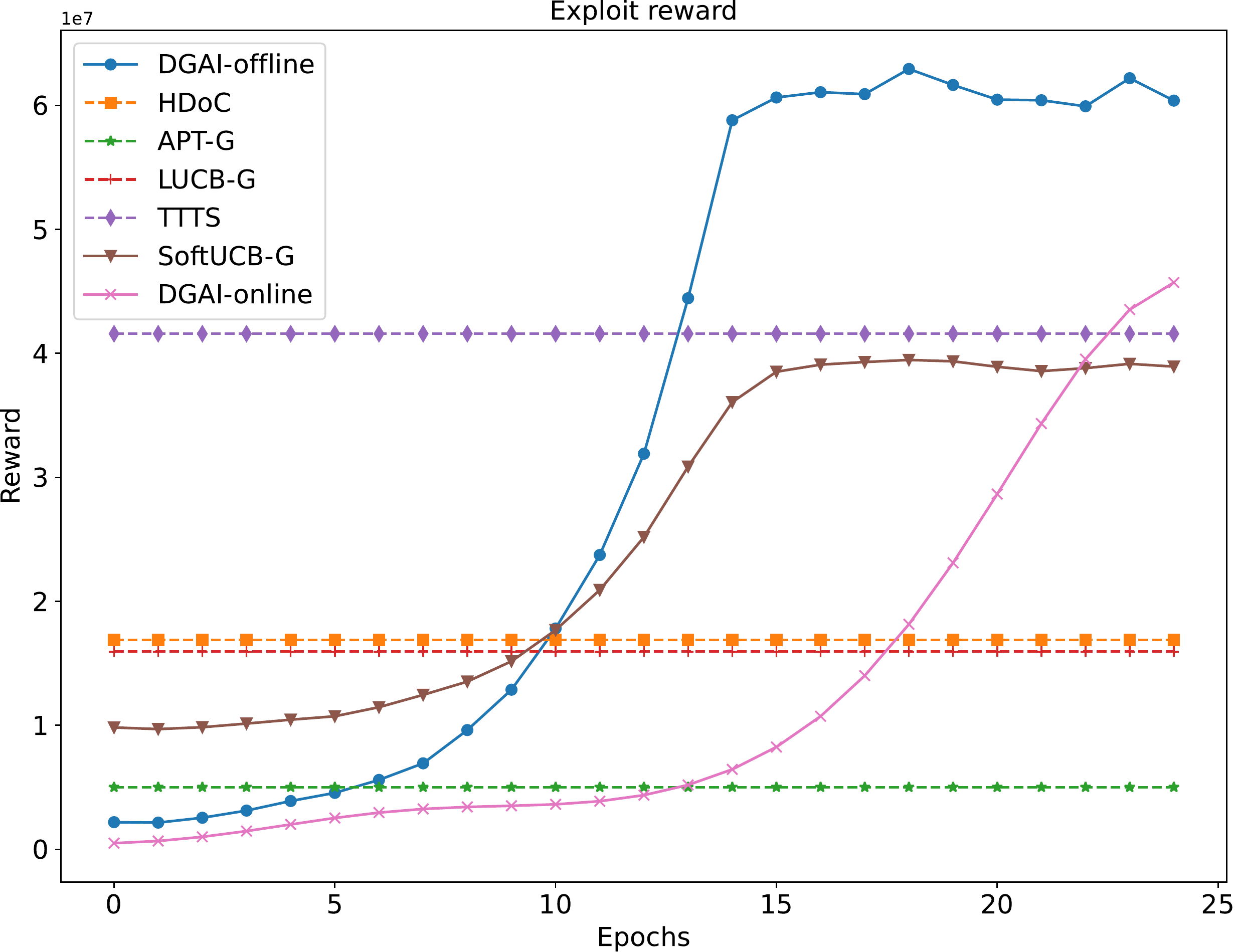}
        \caption{Openbandit}
    \end{subfigure}
    \hfill
    \begin{subfigure}[b]{0.24\textwidth}
        \centering
        \includegraphics[width=\textwidth]{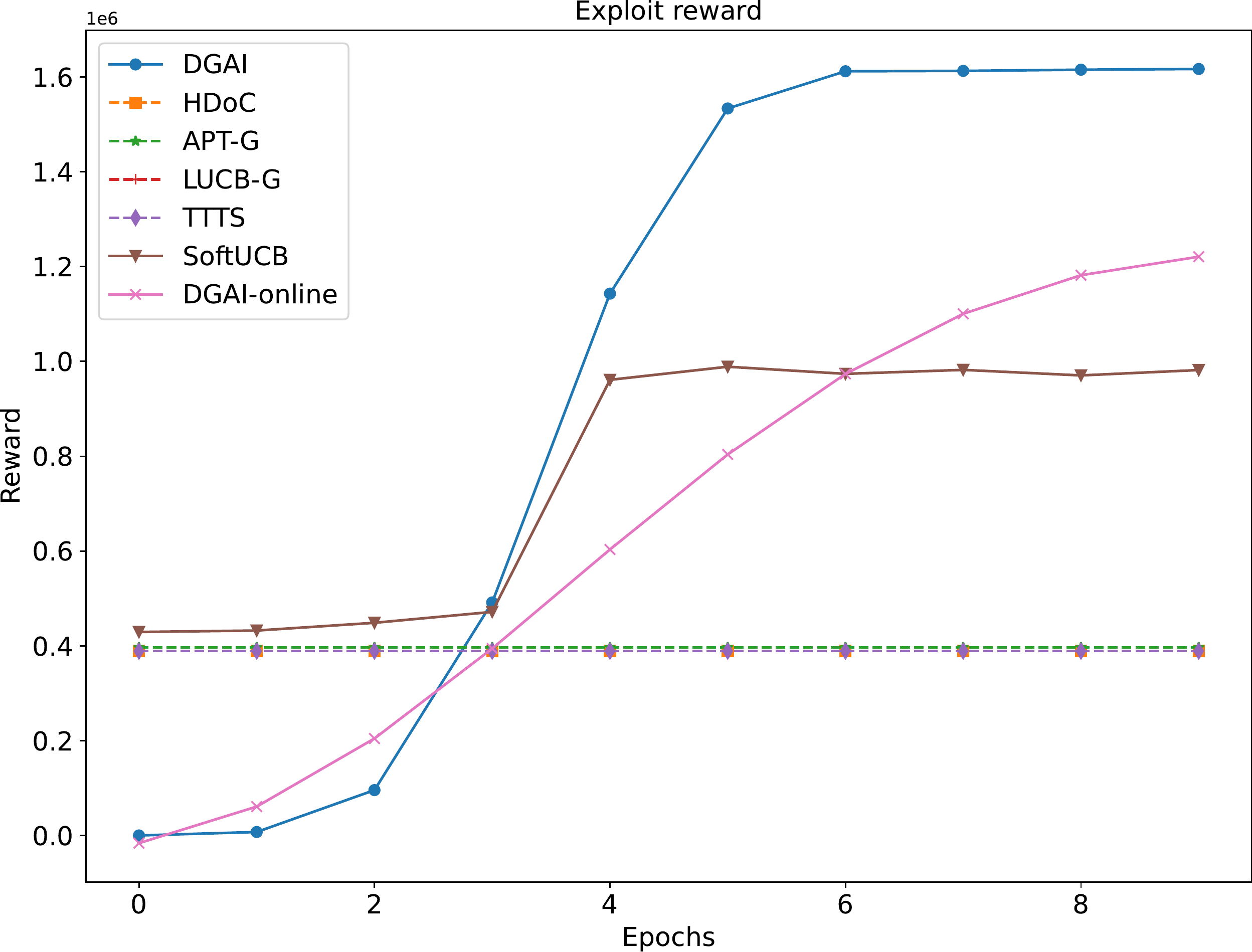}
        \caption{MovieLens}
    \end{subfigure}
    
    
    
    \caption{Comparison of our proposed method with several baselines on all datasets. The cumulative exploit score shows that our proposed method outperforms the baselines in solving GAI problem as the learned confidence bound w.r.t $\alpha,\beta$ converges. The performance in online setting converges slower but also eventually outperform other baselines.}
    \label{fig:exp1}
\end{figure*}

%% file: LaTeX/curve-fig.tex
\begin{figure*}[!ht]
    \centering
    \begin{subfigure}[b]{0.24\textwidth}
        \centering
        \includegraphics[width=\textwidth]{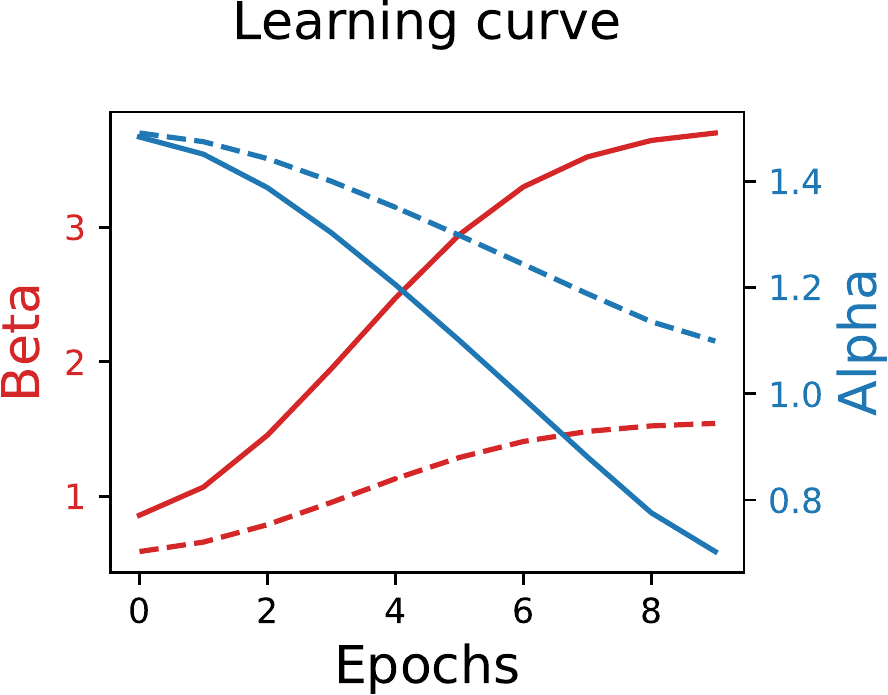}
        \caption{Synth small.}
    \end{subfigure}
    \hfill
    \begin{subfigure}[b]{0.24\textwidth}
        \centering
        \includegraphics[width=\textwidth]{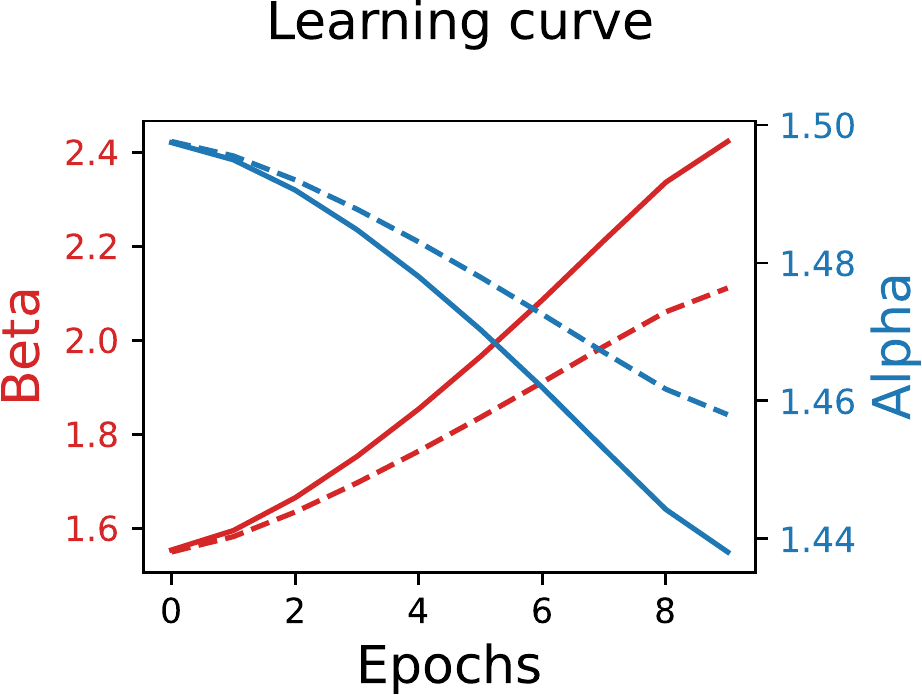}
        \caption{Synth large.}
    \end{subfigure}
    \hfill
    \begin{subfigure}[b]{0.24\textwidth}
        \centering
        \includegraphics[width=\textwidth]{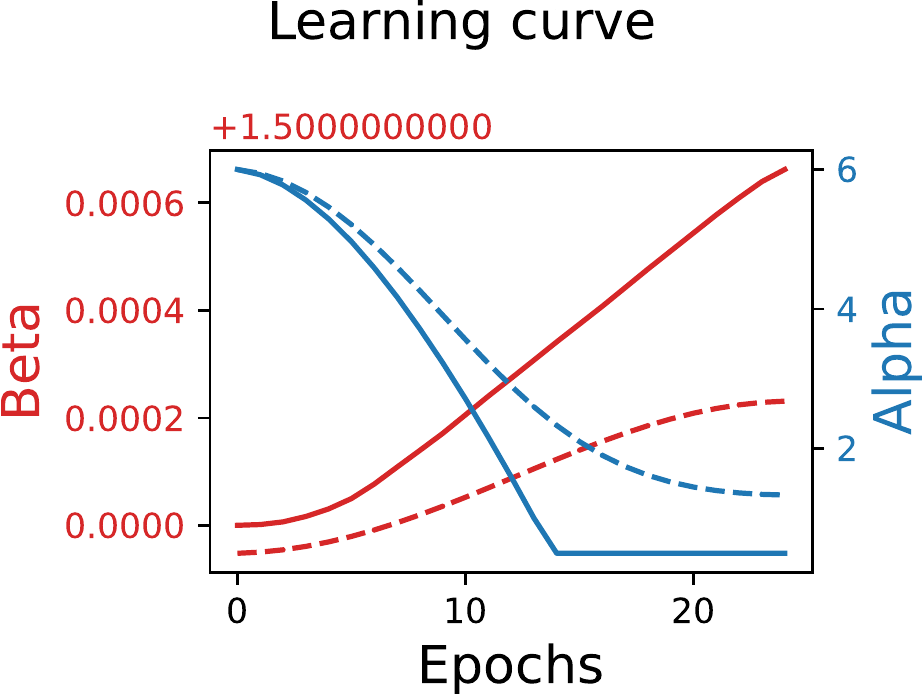}
        \caption{Openbandit.}
    \end{subfigure}
    \hfill
    \begin{subfigure}[b]{0.24\textwidth}
        \centering
        \includegraphics[width=\textwidth]{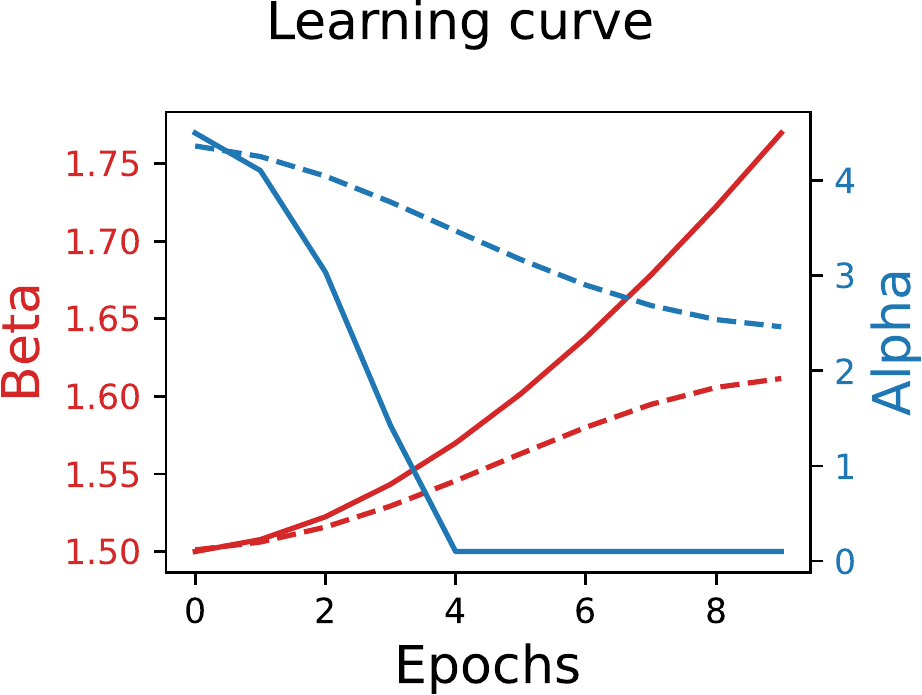}
        \caption{MovieLens.}
    \end{subfigure}
    
    \caption{Learning curves of $\alpha,\beta$. Solid lines represent offline setting and the dash line represent online setting. The horizontal axis in offline setting is training epochs while in online setting is sampling rounds $t\in[T]$. The curve shows that the parameters converges as the training epochs goes on and it converges slower in online setting.}
    \label{fig:learning-curve}
\end{figure*}

%% file: LaTeX/bound-fig.tex
\begin{figure*}[!ht]
    \centering
    \begin{subfigure}[b]{0.23\textwidth}
        \centering
        \includegraphics[width=\textwidth]{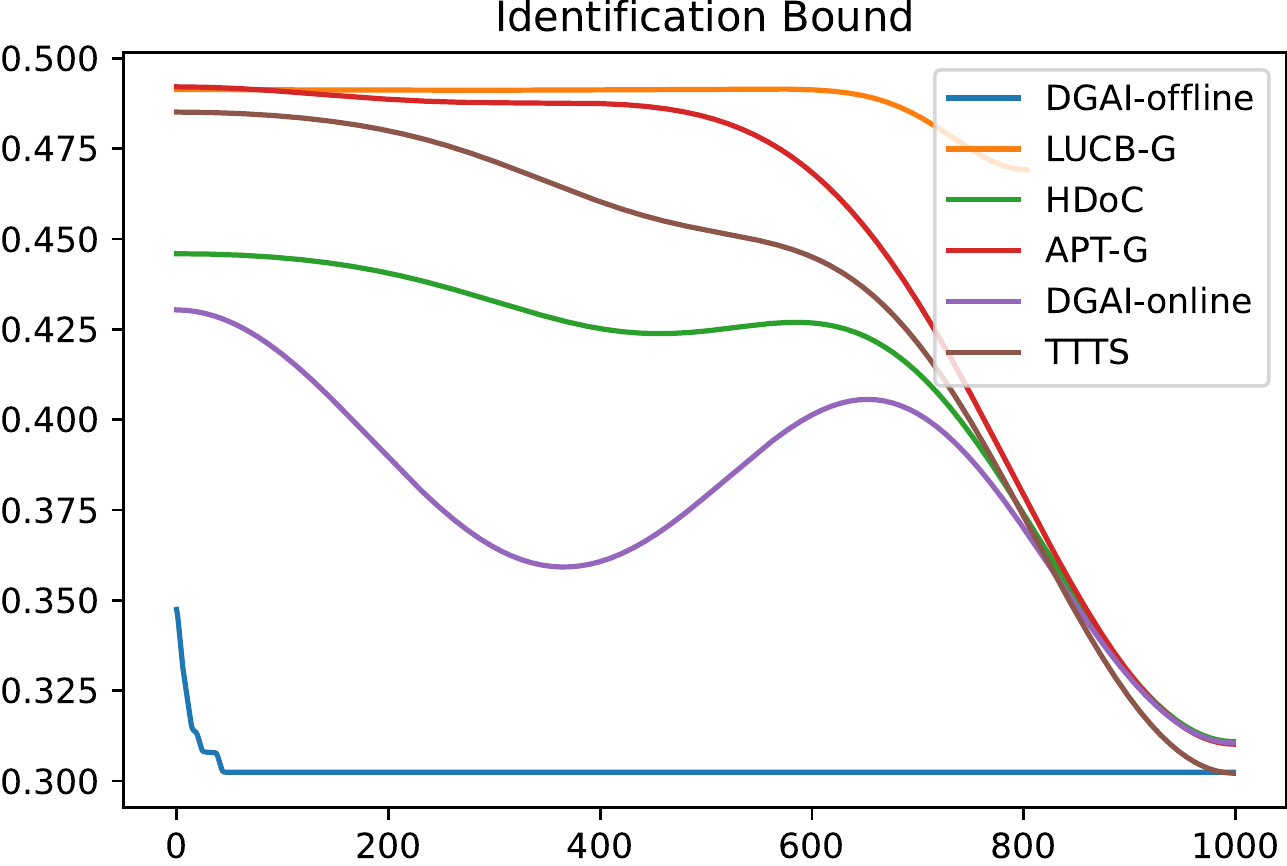}
        \caption{Synth small}
    \end{subfigure}
    \hfill
    \begin{subfigure}[b]{0.23\textwidth}
        \centering
        \includegraphics[width=\textwidth]{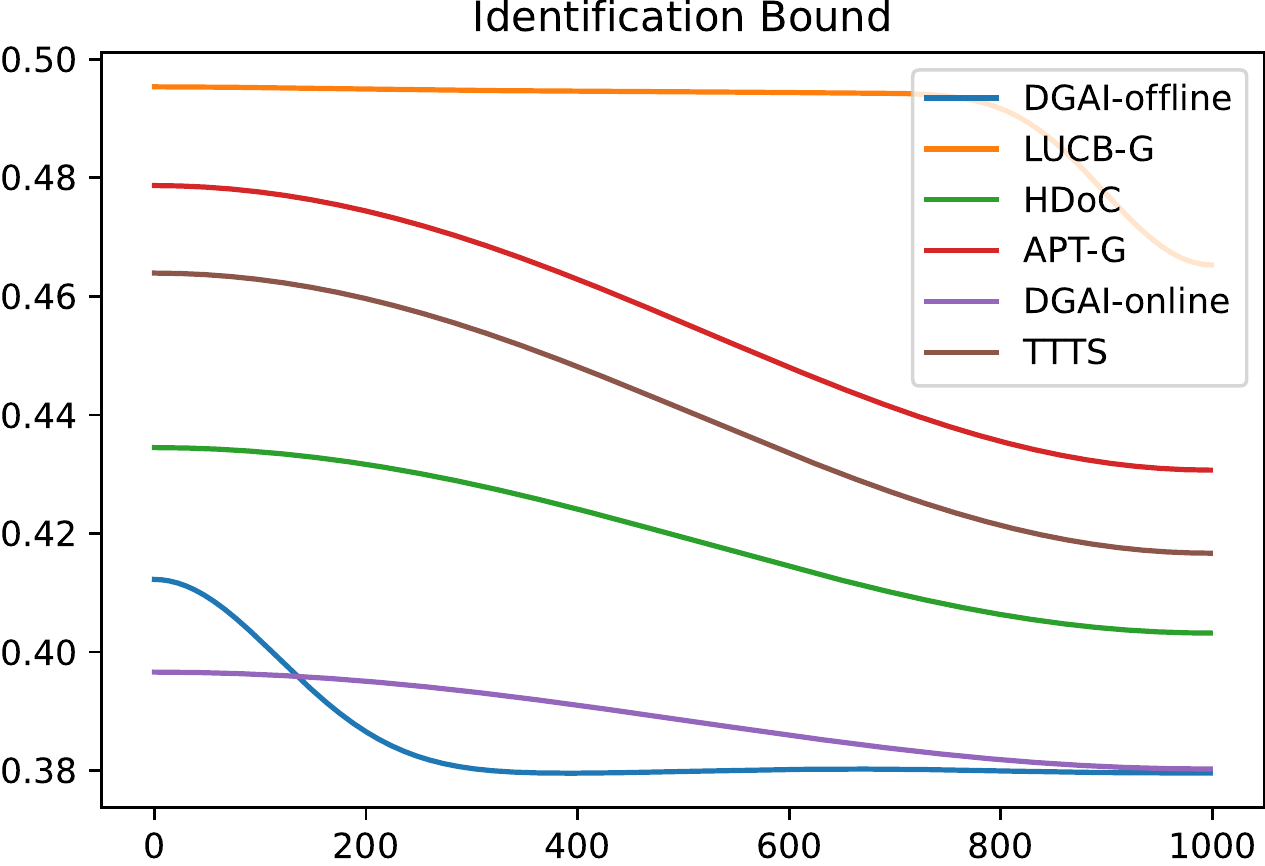}
        \caption{Synth large}
    \end{subfigure}
    \hfill
    \begin{subfigure}[b]{0.23\textwidth}
        \centering
        \includegraphics[width=\textwidth]{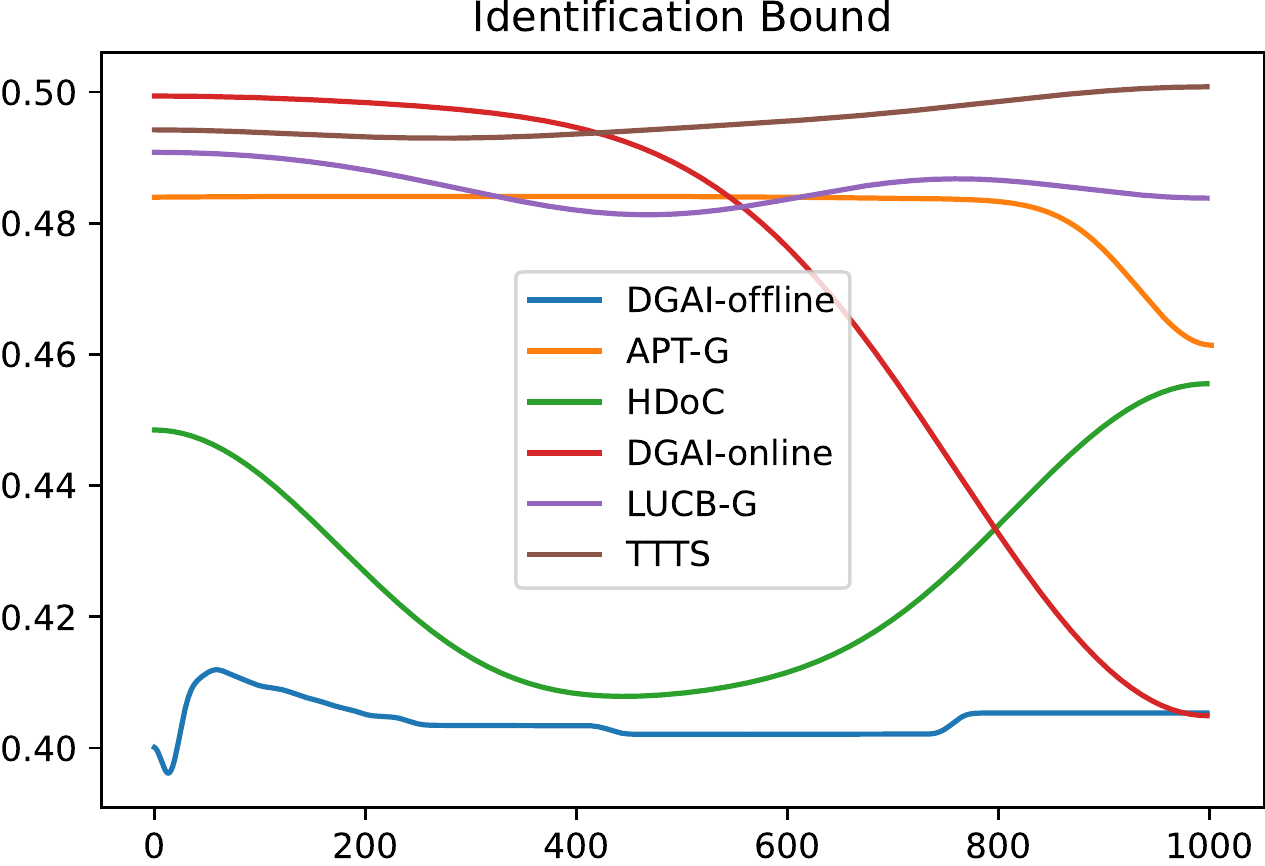}
        \caption{Openbandit}
    \end{subfigure}
    \hfill
    \begin{subfigure}[b]{0.23\textwidth}
        \centering
        \includegraphics[width=\textwidth]{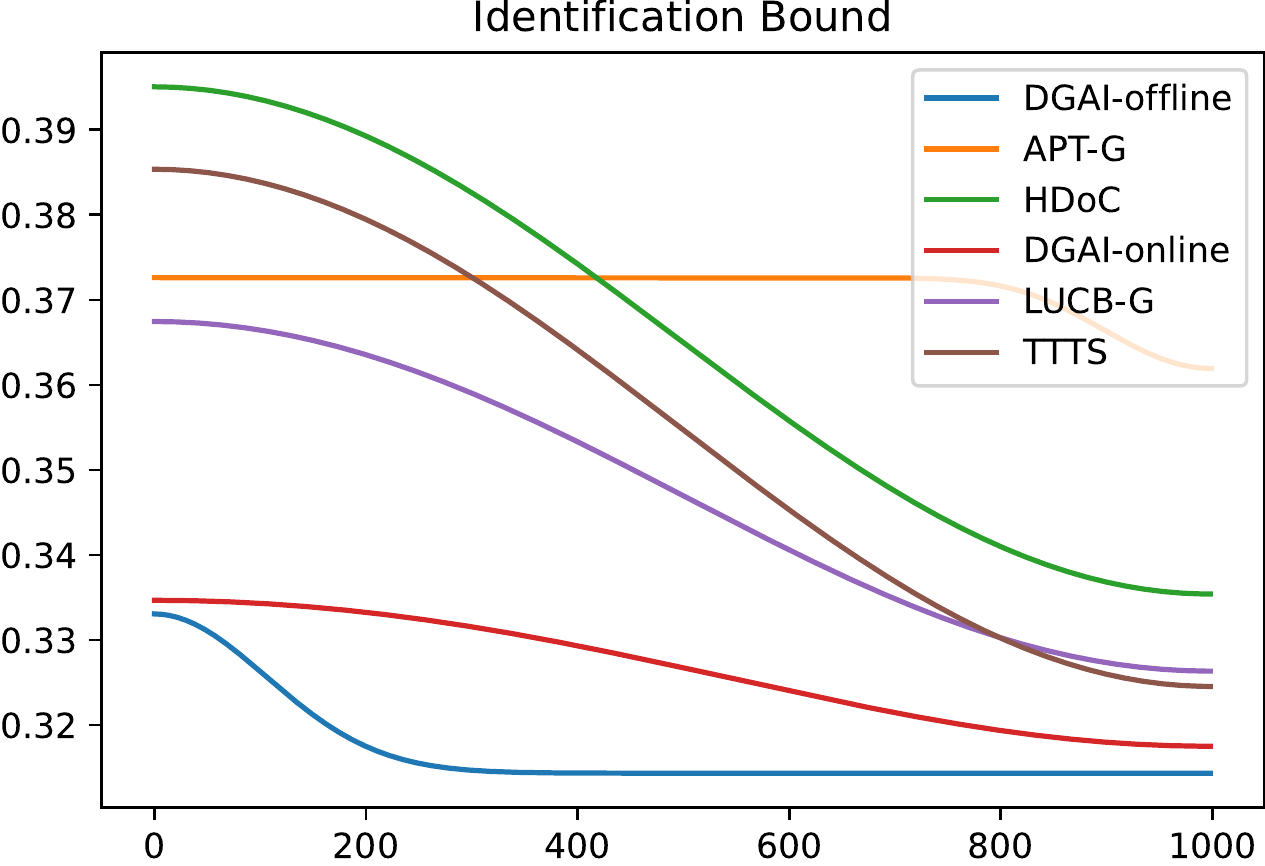}
        \caption{MovieLens}
    \end{subfigure}
    
    \caption{Confidence bound comparison. This figure plots the identification bound for the best arm in the arm set $\mathcal{A}$ during the training trajectory and compare the difference between ours and the baselines.}
    \label{fig:bound-compare}
\end{figure*}

%% file: LaTeX/exp2-performance-fig.tex
\begin{figure*}[t]
    \centering
    \begin{subfigure}[b]{0.23\textwidth}
        \centering
        \includegraphics[width=\textwidth]{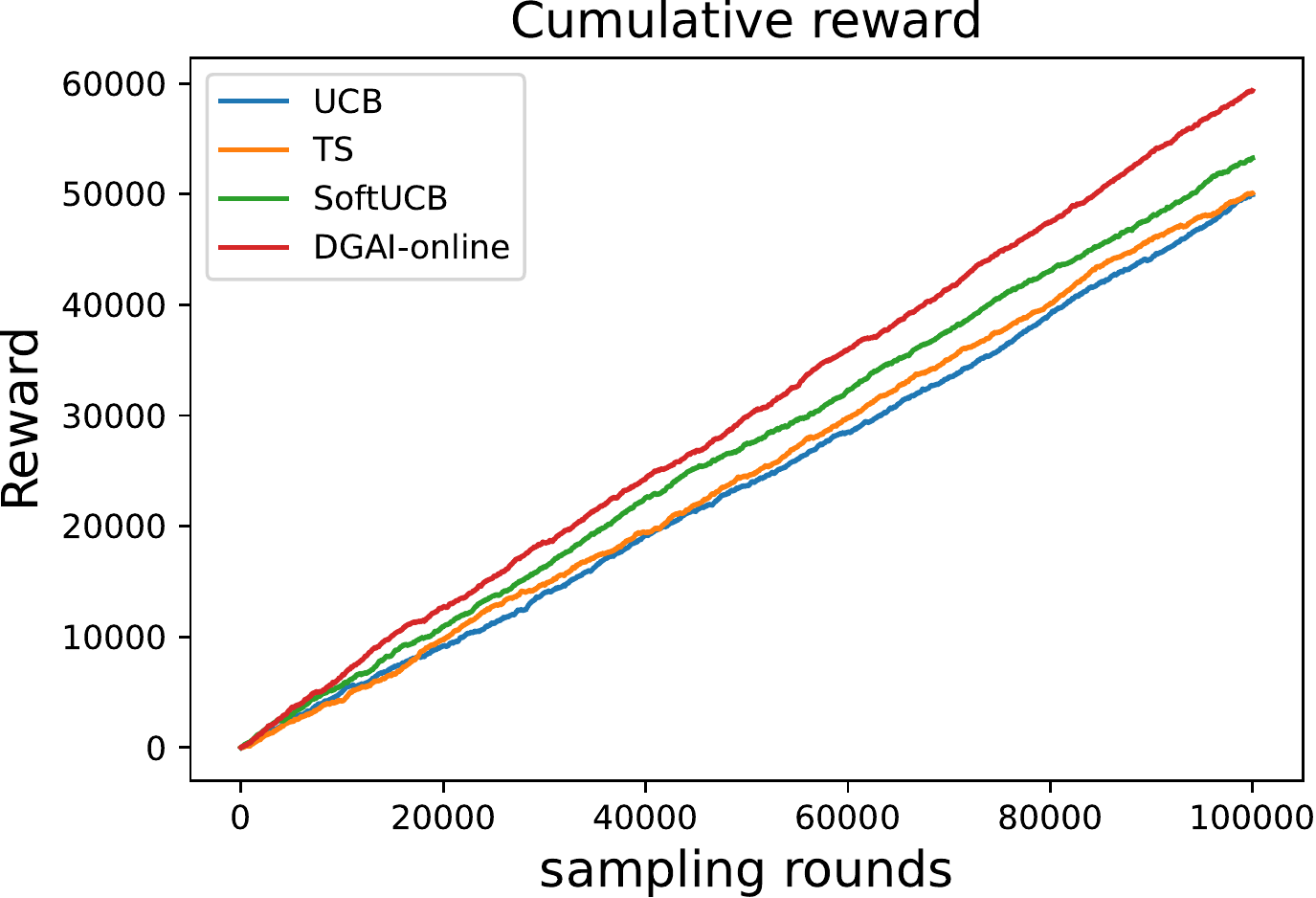}
        \caption{Synth small.}
    \end{subfigure}
    \hfill
    \begin{subfigure}[b]{0.23\textwidth}
        \centering
        \includegraphics[width=\textwidth]{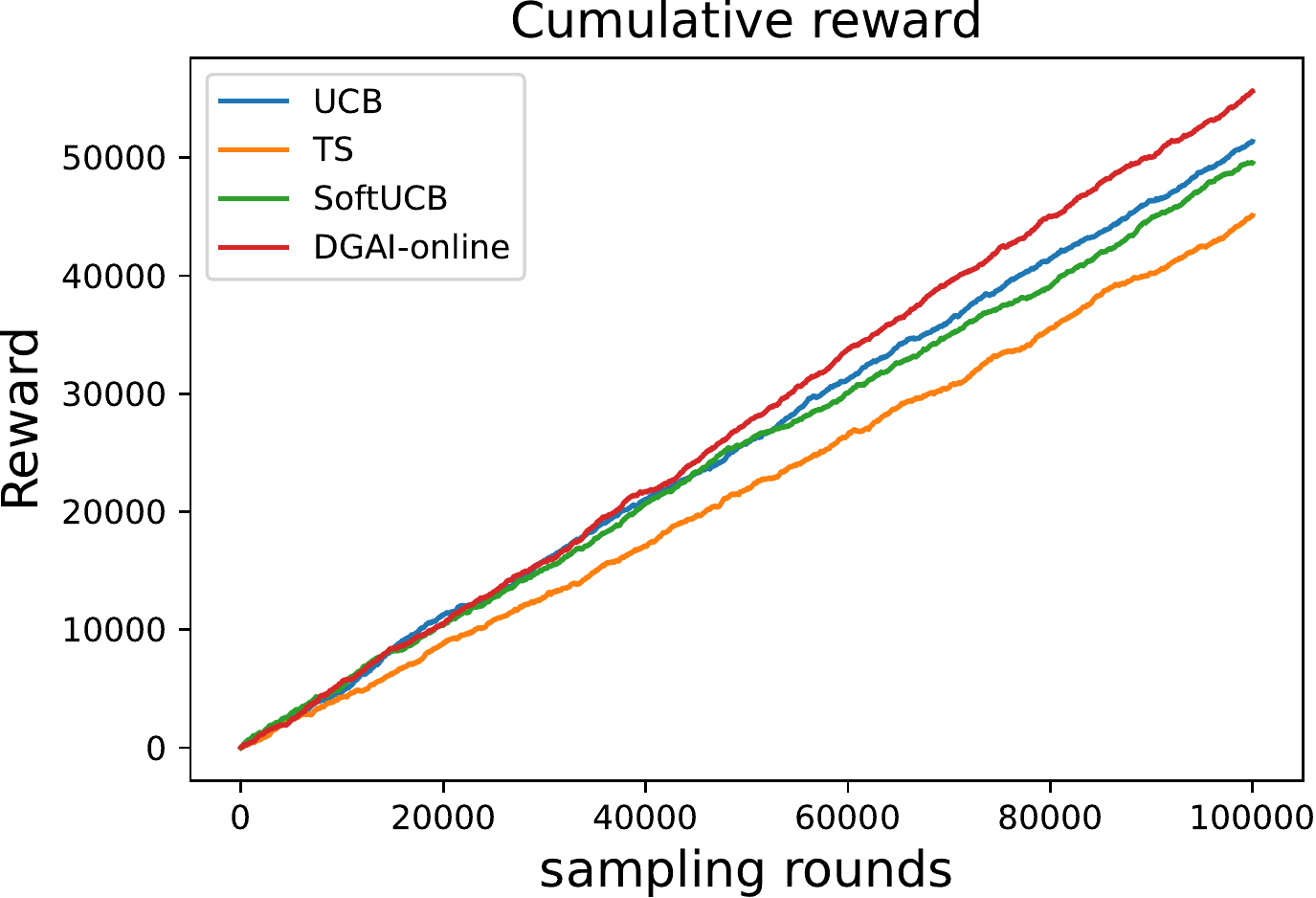}
        \caption{Synth large.}
    \end{subfigure}
    \hfill
    \begin{subfigure}[b]{0.23\textwidth}
        \centering
        \includegraphics[width=\textwidth]{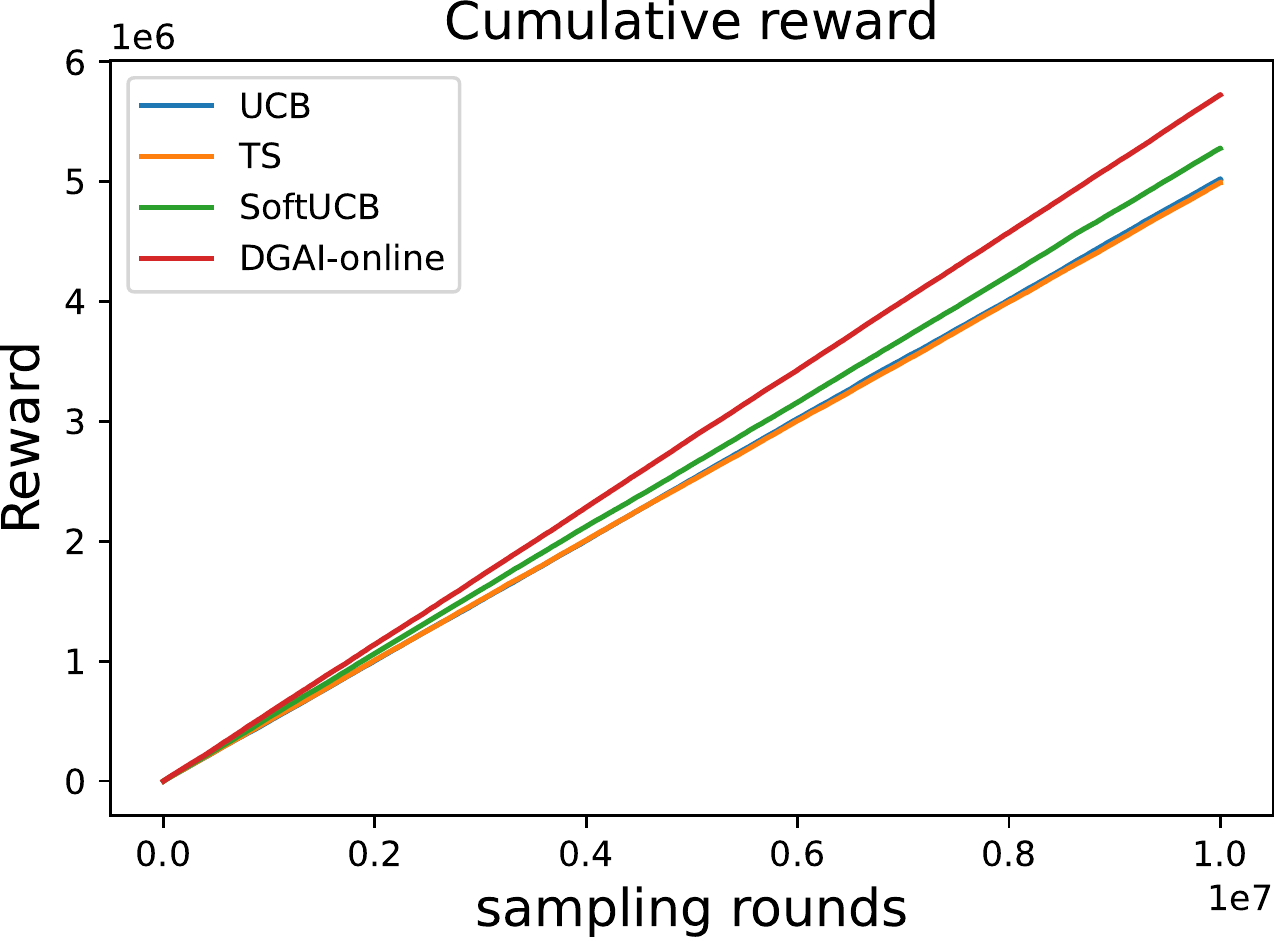}
        \caption{Openbandit.}
    \end{subfigure}
    \hfill
    \begin{subfigure}[b]{0.23\textwidth}
        \centering
        \includegraphics[width=\textwidth]{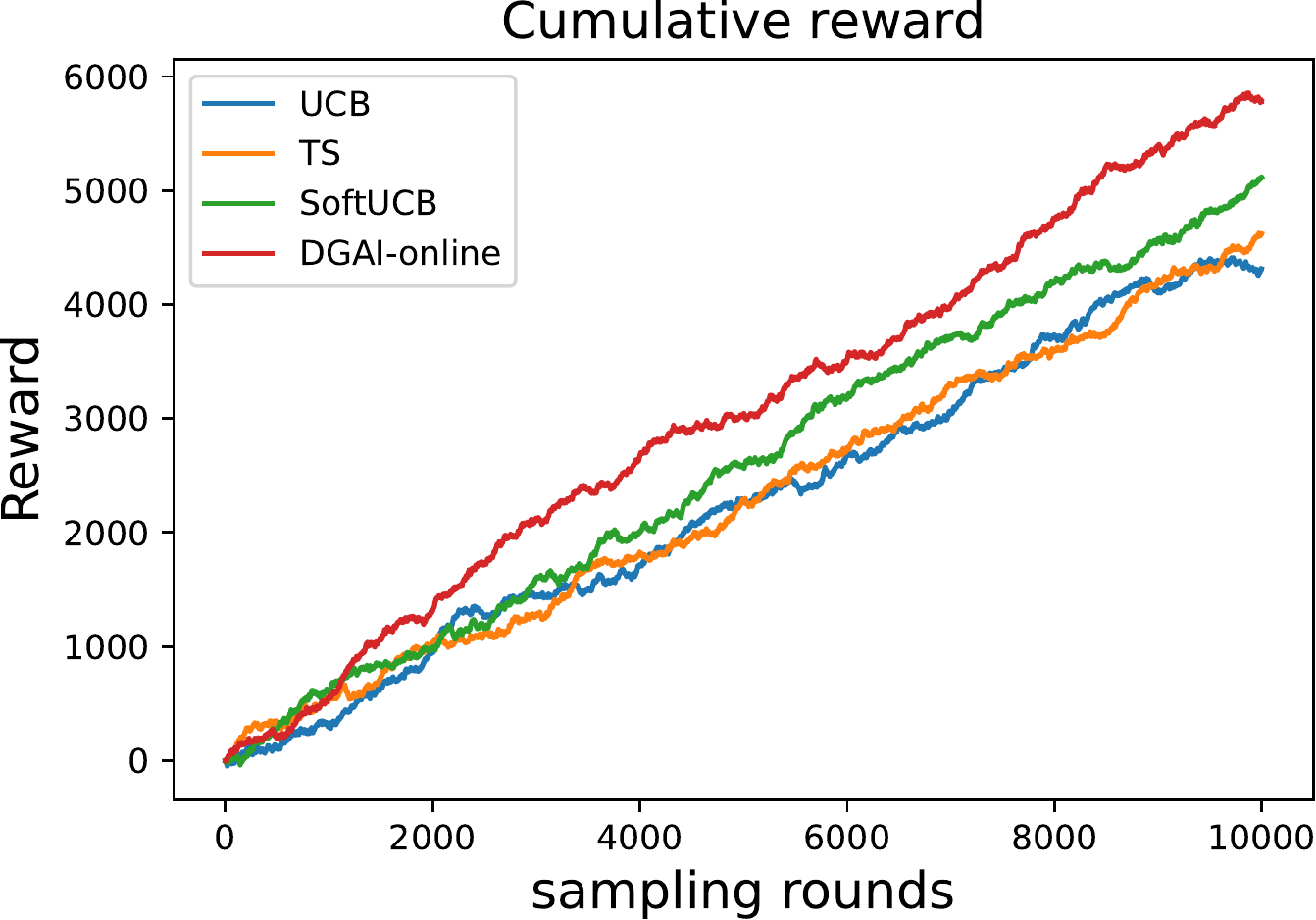}
        \caption{MovieLens.}
    \end{subfigure}
    
    \caption{Comparison of our proposed method with several baselines on all datasets. The cumulative reward shows that our proposed method outperforms the baselines in solving cumulative reward maximization problem as the learned confidence bound w.r.t $\alpha,\beta$ converges to optimal.}
    \label{fig:exp2}
\end{figure*}